\documentclass[10pt,twocolumn]{article}

\usepackage{amssymb}
\usepackage{times}
\usepackage{amsfonts}
\usepackage[T1]{fontenc}
\usepackage{amsthm}
\newtheorem{proposition}{Proposition}

\usepackage[numbers,sort&compress]{natbib}

\usepackage{amsmath, amsthm}
\usepackage{mathtools}
\usepackage{microtype}
\usepackage{bm}
\usepackage{enumitem}
\usepackage{subcaption}
\newtheorem{assumption}{Assumption}
\usepackage{hyperref}

\newtheorem{definition}{Definition}

\theoremstyle{remark}
\usepackage{booktabs}     
\usepackage{multirow}    
\usepackage{array}

\usepackage{framed,xcolor}
\usepackage{graphicx}
\usepackage{tabularx}
\definecolor{PCAccent}{RGB}{85,78,68}
\definecolor{PCBorder}{RGB}{227,222,214}
\definecolor{PCBack}{RGB}{252,250,246}
\definecolor{PCTitle}{RGB}{38,36,33}
\definecolor{PCBand}{RGB}{246,243,238}

\definecolor{PCAccent}{gray}{0.25}   
\definecolor{PCBorder}{gray}{0.70}   
\definecolor{PCBack}{gray}{0.98}     
\definecolor{PCTitle}{gray}{0.10}    
\definecolor{PCBand}{gray}{0.94}     

\title{Knowing When Not to Answer: Abstention-Aware Scientific Reasoning}
\author{
Samir Abdaljalil\\
Texas A\&M University\\
\texttt{sabdjalil@tamu.edu}
\and
Erchin Serpedin\\
Texas A\&M University\\
\texttt{eserpedin@tamu.edu}
\and
Hasan Kurban\\
Hamad Bin Khalifa University\\
\texttt{hkurban@hbku.edu.qa}
}
\date{}

\renewenvironment{abstract}
 {
  \centerline{\bfseries Abstract}
  \vspace{0.5em}
  \begin{quote}
  \small
 }
 {
  \end{quote}
  \vspace{1em}
 }

\begin{document}

\maketitle
\begin{abstract}
Large language models are increasingly used to answer and verify scientific claims, yet existing evaluations typically assume that a model must always produce a definitive answer. In scientific settings, however, unsupported or uncertain conclusions can be more harmful than abstaining. We study this problem through an abstention-aware verification framework that decomposes scientific claims into minimal conditions, audits each condition against available evidence using natural language inference (NLI), and selectively decides whether to support, refute, or abstain.
We evaluate this framework across two complementary scientific benchmarks: SciFact and PubMedQA, covering both closed-book and open-domain evidence settings. Experiments are conducted with six diverse language models, including encoder-decoder, open-weight chat models, and proprietary APIs. Across all benchmarks and models, we observe that raw accuracy varies only modestly across architectures, while abstention plays a critical role in controlling error. In particular, confidence-based abstention substantially reduces risk at moderate coverage levels, even when absolute accuracy improvements are limited.
Our results suggest that in scientific reasoning tasks, the primary challenge is not selecting a single best model, but rather determining when available evidence is sufficient to justify an answer. This work highlights abstention-aware evaluation as a practical and model-agnostic lens for assessing scientific reliability, and provides a unified experimental basis for future work on selective reasoning in scientific domains. Code is available at \\
\url{https://github.com/sabdaljalil2000/ai4science}.
\end{abstract}
\section{Introduction}
\label{sec:intro}
Large language models are increasingly used in scientific workflows to answer questions, assess factual claims, and summarize evidence \cite{yildiz2025,zheng-etal-2025-automation}. In many deployed settings, these systems are treated as conventional predictors that are expected to always produce a definitive output. However, scientific reasoning differs fundamentally from many standard prediction tasks \cite{liu2026wildsciadvancingscientificreasoning}. In scientific contexts, an incorrect conclusion can be more harmful than providing no conclusion at all, particularly when evidence is incomplete, ambiguous, or weak. Despite this, most existing evaluations of scientific reasoning systems focus exclusively on unconditional accuracy \cite{rueda2025understandingllmscientificreasoning}. Models are scored based on whether they return the correct label or answer, with no explicit mechanism for expressing uncertainty or declining to answer. As a result, systems are implicitly incentivized to guess even when available evidence does not justify a confident decision. This evaluation paradigm obscures a critical aspect of scientific reliability: the ability to recognize when evidence is insufficient.

In this work, we study scientific claim verification and question answering through the lens of abstention-aware decision making. Rather than assuming that a system must always answer, we explicitly model the option to abstain. Our approach decomposes scientific inputs into minimal conditions, audits each condition against evidence using NLI, and aggregates these audits into a final decision that may support, refute, or abstain. This structure allows the system to separate evidence assessment from decision confidence, and enables principled control over the tradeoff between error and coverage. We evaluate this framework on two widely used scientific benchmarks, SciFact \cite{wadden-etal-2020-fact} and PubMedQA \cite{jin-etal-2019-pubmedqa}. Our experimental results reveal three consistent patterns. First, across models and datasets, unconditional accuracy varies within a relatively narrow range. Larger or proprietary models do not reliably dominate smaller or open models under standard accuracy metrics. Second, confidence-based abstention dramatically reduces error at moderate coverage levels. Even when absolute accuracy improvements are limited, selectively withholding low-confidence predictions leads to substantial reductions in risk. Third, these trends are remarkably stable across model families, suggesting that abstention behavior is more predictive of reliability than architectural differences. These findings point to a broader conclusion. In scientific reasoning tasks, the primary challenge is often not selecting the single best model, but determining when available evidence is sufficient to justify an answer. Our results suggest that abstention-aware evaluation provides a more informative and practically relevant lens for assessing scientific reasoning systems than accuracy alone.

This paper makes three contributions: (i) we introduce an absten-tion-aware verification and question answering pipeline that decomposes reasoning into auditable conditions and supports selective prediction; (ii) we provide a systematic empirical study across multiple scientific benchmarks and model families, demonstrating that risk control through abstention is a dominant factor in scientific reliability; and (iii) we offer a unified evaluation framework based on risk–coverage analysis that exposes behaviors hidden by standard metrics. Together, these contributions argue for a shift in how scientific reasoning systems are evaluated: rather than asking whether a model can answer correctly, we advocate for asking whether it can recognize when it should not answer at all.


\section{Background and Related Works}
\label{sec:related}
\noindent
\textbf{Scientific Reasoning and Claim Verification.}
Automated reasoning over scientific text has received growing attention in recent years \cite{lin-etal-2025-llm,yin2025scientificreasoningllmstraining}. Benchmarks such as SciFact and PubMedQA were introduced to evaluate whether models can assess the validity of scientific claims or answer biomedical questions using textual evidence. These datasets emphasize evidence grounding and scientific rigor, distinguishing them from general fact checking benchmarks. Prior work has explored neural architectures for joint evidence selection and claim classification \cite{SAPUTRA2024100489,adhikari-etal-2019-rethinking}, as well as the use of NLI to align claims with supporting or contradicting sentences \cite{gemechu-etal-2025-natural,mirallesgonzalez2025pushingboundarynaturallanguage,kim-etal-2023-conditional}. While these efforts have advanced performance on benchmark metrics, most evaluations assume that a model must always produce a definitive answer, even in the presence of ambiguous or incomplete evidence \cite{koupaee-etal-2025-faithful}.\\ 
\noindent
\textbf{Uncertainty, Abstention, and Selective Prediction.}
The problem of deciding when to abstain has a long history in statistical learning under the framework of selective prediction. Some work formalizes the tradeoff between coverage and risk, showing that allowing a model to withhold predictions can strictly improve reliability under appropriate confidence estimates \cite{kapoor2024,yin-etal-2023-large}. More recent work has revisited selective prediction in neural models, focusing on calibration, confidence scoring, and risk control \cite{parikh2026cattobalancingpreferencesconfidence}. In parallel, abstention has been studied in safety-critical domains such as medical decision support and information retrieval. However, in the context of scientific reasoning with language models, abstention is rarely treated as a first-class evaluation dimension. Instead, uncertainty is often assessed indirectly through calibration metrics or ignored entirely in favor of accuracy.
 \\
\noindent
\textbf{Abstention in LLMs.}
Some recent literature has begun to study abstention behavior in LLMs. More broadly, abstention has been studied in general domain QA and language modeling through prompt-based strategies \cite{10.1145/3477495.3532048}, uncertainty estimation\cite{yin-etal-2023-large}, and selective generation frameworks\cite{asai-choi-2021-challenges}. Our approach treats abstention as a first-class evaluation dimension for scientific reasoning, emphasizing selective risk over unconditional accuracy. This complements prior analyses by providing a unified pipeline that connects decomposition, evidence auditing, and abstention-aware evaluation across both verification and biomedical QA tasks.
 \\
\noindent
\textbf{Evaluation of Language Models in Scientific Domains.}
Several studies have examined the behavior of large language models on scientific benchmarks, often comparing performance across model sizes and architectures \cite{song2025evaluatinglargelanguagemodels,zhengNature2025,doi:10.1126/science.ads0018}. These evaluations typically report accuracy or F1 scores and focus on identifying the strongest model for a given task. Some work has analyzed error types or hallucination tendencies \cite{asgari2025}, but such analyses are usually qualitative or post hoc. Most closely related is the work of Wen et al.~\cite{wen-etal-2024-characterizing}, who characterize how models abstain in science question answering under controlled context perturbations. They examine how removing, replacing, or augmenting context affects abstention rates across multiple scientific QA datasets, and show that abstention behavior varies significantly by model architecture, question type, and prompt formulation. Our work differs in both problem formulation and evaluation focus. Rather than probing abstention through external context perturbations, we model abstention as an explicit decision variable within a structured verification pipeline. We decompose inputs into auditable conditions, verify each condition against evidence using a separate NLI verifier, and derive abstention decisions through selective risk control. This allows us to study abstention without modifying the input distribution and to evaluate reliability through risk–coverage analysis rather than abstention rate alone. While Wen et al.~\cite{wen-etal-2024-characterizing} focus on when models abstain under changing context, our work focuses on why abstention occurs, grounding decisions in condition-level evidence sufficiency and providing a decision-theoretic interpretation of scientific reliability.


\section{Methodology}
\label{sec:method}

This section establishes the formal framework for selective classification (Section~\ref{sec:formulation}), then describes each pipeline stage: condition decomposition (Section~\ref{sec:decomposition}), evidence auditing via NLI (Section~\ref{sec:audit}), decision aggregation under asymmetric loss (Section~\ref{sec:decision}), and confidence-based abstention (Section~\ref{sec:confidence}). We then present the risk--coverage evaluation protocol (Section~\ref{sec:rc}). Formal theoretical results and proofs are provided in Appendix~\ref{app:theory}. An illustration of the methodology is presented in Fig. \ref{fig:system_overview}.

\subsection{Problem Formulation and Selective Classification}
\label{sec:formulation}

Let $x \in \mathcal{X}$ denote a scientific input, which may be a factual claim (as in SciFact) or a question (as in PubMedQA). Let $\mathcal{E} = \{e_1, \dots, e_n\} \in \mathbb{E}$ denote a set of textual evidence sentences associated with $x$, drawn from the space of all possible evidence sets $\mathbb{E}$. The task is to produce an output $y$ from a discrete label set $\mathcal{Y}$, or to abstain when evidence is insufficient. Unlike standard classification, the prediction space is extended to $\mathcal{Y} \cup \{\bot\}$, where $\bot$ denotes abstention. We formalize this within the selective classification framework~\cite{el2010foundations,geifman2017selective}.

\begin{definition}[Selective Classifier]
\label{def:selective}
A selective classifier is a pair $(F, g)$, where $F\colon \mathcal{X} \times \mathbb{E} \to \mathcal{Y}$ is an end-to-end prediction function and $g\colon \mathcal{X} \times \mathbb{E} \to \{0, 1\}$ is a selection function. The system outputs $F(x, \mathcal{E})$ when $g(x, \mathcal{E}) = 1$ and abstains ($\bot$) when $g(x, \mathcal{E}) = 0$.
\end{definition}

\begin{definition}[Selective Risk and Coverage]
\label{def:risk_coverage}
Given a loss function $\ell\colon \mathcal{Y} \times \mathcal{Y} \to \mathbb{R}_{\geq 0}$ and a data distribution $\mathcal{D}$ over $(x, \mathcal{E}, y)$, the \emph{selective risk} and \emph{coverage} of $(F, g)$ are defined for any $(F,g)$ with $\phi(F,g)>0$ as:
\begin{equation}
\label{eq:selective_risk}
R(F, g) \;=\; \frac{\mathbb{E}_{(x, \mathcal{E}, y) \sim \mathcal{D}}\!\big[\ell\!\big(F(x, \mathcal{E}),\, y\big) \cdot g(x, \mathcal{E})\big]}{\mathbb{E}_{(x, \mathcal{E}, y) \sim \mathcal{D}}\!\big[g(x, \mathcal{E})\big]},
\end{equation}
\begin{equation}
\label{eq:coverage}
\phi(F, g) \;=\; \mathbb{E}_{(x, \mathcal{E}, y) \sim \mathcal{D}}\!\big[g(x, \mathcal{E})\big].
\end{equation}
\end{definition}

The objective is not to maximize unconditional accuracy, but to minimize selective risk subject to a controllable coverage constraint. This framing captures the core scientific requirement: errors among committed predictions should be minimized, even at the cost of answering fewer queries. In our pipeline, the end-to-end prediction function $F$ is realized through condition decomposition, evidence auditing, and rule-based aggregation (Sections~\ref{sec:decomposition}--\ref{sec:decision}). Specifically, let $k = |C|$ denote the number of decomposed conditions (Section~\ref{sec:decomposition}), let $a_i \in \mathcal{A} = \{\textsc{sup}, \textsc{con}, \textsc{mis}\}$ denote the audit result for condition $c_i$ (Section~\ref{sec:audit}), and let $\beta_i \in \{0,1\}$ indicate whether $c_i$ is critical. The end-to-end prediction is:
\begin{equation}
\label{eq:F_decomposition}
F(x, \mathcal{E}) \;=\; h\!\big(a_1(x, \mathcal{E}),\, \dots,\, a_k(x, \mathcal{E}),\; \beta_1, \dots, \beta_k\big),
\end{equation}
where $h$ is a deterministic aggregation function (Section~\ref{sec:decision}). We write $a_i(x, \mathcal{E})$ to denote the audit outcome as a function of the input and evidence; formally, $a_i(x, \mathcal{E}) = a_i(c_i(x), \mathcal{E})$, where $c_i(x)$ is the $i$-th condition produced by decomposing $x$. Since our theoretical analysis conditions on a fixed decomposition (Section~\ref{sec:decomposition}), $c_i$ is determined by $x$ and this composition is well-defined. The selection function $g$ is defined by thresholding an explicit confidence score (Section~\ref{sec:confidence}):
\begin{equation}
\label{eq:selection}
g_\tau(x, \mathcal{E}) \;=\; \mathbf{1}\!\big[\mathrm{conf}(x, \mathcal{E}) \ge \tau\big],
\end{equation}
where $\tau \in [0, 1]$ controls the coverage--risk tradeoff. This separation ensures that the decision of \emph{what to predict} and the decision of \emph{whether to predict} are governed by distinct, independently interpretable mechanisms.
\\
\noindent
\textbf{Notational conventions.} For brevity, we write $\mathrm{conf}(x)$ as shorthand for $\mathrm{conf}(x, \mathcal{E})$ and $g_\tau(x)$ for $g_\tau(x, \mathcal{E})$ when the associated evidence set is clear from context. Similarly, $\phi(\tau) \coloneqq \phi(F, g_\tau) = \mathbb{E}_{\mathcal{D}}[g_\tau(x, \mathcal{E})]$ denotes the coverage at threshold $\tau$. Unless otherwise stated, $\ell$ denotes the 0--1 loss, $\ell(\hat{y}, y) = \mathbf{1}[\hat{y} \neq y]$, which we use throughout evaluation; all formal results specify their loss assumptions explicitly.

\subsection{Condition Decomposition}
\label{sec:decomposition}

The first stage of the pipeline decomposes the input $x$ into a set of $k$ minimal conditions:
\begin{equation}
\label{eq:conditions}
C = \{c_1, c_2, \dots, c_k\},
\end{equation}
where each condition $c_i$ represents a necessary factual requirement for a valid conclusion. Conditions are expressed in natural language and may be marked as \emph{critical} ($\beta_i = 1$), indicating that failure of the condition alone is sufficient to invalidate the conclusion. Let $C_{\mathrm{crit}} = \{c_i \in C : \beta_i = 1\}$ denote the subset of critical conditions. We require $|C_{\mathrm{crit}}| \geq 1$; every decomposition must identify at least one critical condition. This ensures that the confidence score (Section~\ref{sec:confidence}) and aggregation rules (Section~\ref{sec:decision}) are well-defined. Condition decomposition is performed by a language model prompted to rewrite the input into atomic, testable statements. For example, a claim asserting a causal relationship may be decomposed into conditions concerning population, intervention, outcome, and directionality. This step reduces complex scientific reasoning into simpler sub-problems that can be independently audited against evidence, and introduces an explicit intermediate representation that makes the reasoning process inspectable. A probabilistic factorization view of condition-level auditing, along with the audit-independence assumption and its implications, is provided in Appendix~\ref{sec:appendix_factorization}.

\subsection{Condition Auditing via Natural Language Inference}
\label{sec:audit}

Each condition $c_i$ is audited against the evidence set $\mathcal{E}$ using a pretrained NLI model. For each condition--evidence pair $(c_i, e_j)$, the NLI model outputs a probability distribution over entailment, contradiction, and neutrality. For a given condition $c_i$, we aggregate evidence by selecting the most supportive and most contradictory sentences:
\begin{equation}
\label{eq:evidence_scores}
s_i^{\mathrm{ent}} = \hspace*{-2mm} \max_{j \in \{1,\dots,n\}} P(\text{entail} \mid c_i, e_j), \
s_i^{\mathrm{con}} =  \hspace*{-2mm} \max_{j \in \{1,\dots,n\}} P(\text{contradict} \mid c_i, e_j).
\end{equation}

Note that $s_i^{\mathrm{ent}}$ and $s_i^{\mathrm{con}}$ may be maximized by different evidence sentences, since the maximum is taken independently for each relation. Based on fixed thresholds $(\theta_{\mathrm{ent}}, \theta_{\mathrm{con}})$, each condition is assigned a discrete audit status $a_i \in \mathcal{A}$. Consistent with the asymmetric loss structure formalized in Section~\ref{sec:decision} (Assumption~\ref{assump:loss}), contradiction takes priority when both thresholds are simultaneously exceeded:
\begin{equation}
\label{eq:audit_status}
a_i =
\begin{cases}
\textsc{con} & \text{if } s_i^{\mathrm{con}} \geq \theta_{\mathrm{con}} \\[2pt]
\textsc{sup} & \text{if } s_i^{\mathrm{ent}} \geq \theta_{\mathrm{ent}} \;\text{and}\; s_i^{\mathrm{con}} < \theta_{\mathrm{con}} \\[2pt]
\textsc{mis} & \text{otherwise}.
\end{cases}
\end{equation}

Neutral predictions are treated as absence of evidence rather than weak support, reflecting scientific norms where lack of evidence does not imply correctness. The NLI verifier is a DeBERTa-based cross-encoder trained on large-scale general-domain NLI datasets, including SNLI and MultiNLI. It is not fine-tuned on SciFact or PubMedQA, ensuring that auditing remains decoupled from downstream task labels. While explicit probability calibration is not performed on the NLI outputs, thresholds are fixed across datasets and models to isolate the effects of pipeline structure.
\\ \\
\noindent
\textbf{Choice of NLI Model.}
We employ a general-domain NLI model, rather than a biomedical-specific model since
the verifier is used exclusively to assess textual entailment or contradiction between a condition and a candidate evidence sentence, not to perform domain-specific scientific reasoning. 
Since our analysis focuses on selective risk--coverage behavior rather than absolute task performance, this choice provides a consistent and safety-oriented evaluation setting.
\\ \\
\noindent
\textbf{Threshold selection.}
Entailment and contradiction thresholds $(\theta_{\mathrm{ent}},$
$ \theta_{\mathrm{con}})$ are fixed across all experiments. Thresholds are selected conservatively to prioritize precision over recall, reflecting the asymmetric cost of false support in scientific settings. We do not tune thresholds per dataset or per model. Instead, thresholds are chosen once and held constant to ensure comparability across models and tasks. This design avoids implicit overfitting to benchmark-specific distributions and isolates the effect of abstention behavior from threshold optimization. Implementation and computational details are provided in Appendix~\ref{sec:appendix_complexity}.

\subsection{Decision Aggregation Under Asymmetric Loss}
\label{sec:decision}

Audited conditions are aggregated into a final task-specific decision by the deterministic aggregation function:
\begin{equation}
\label{eq:aggregation_type}
h\colon \mathcal{A}^k \times \{0,1\}^k \;\to\; \mathcal{Y},
\end{equation}
which takes the vector of audit outcomes $\mathbf{a} = (a_1, \dots, a_k)$, as input, together with criticality indicators $\boldsymbol{\beta} = (\beta_1, \dots, \beta_k)$. Non-critical conditions ($\beta_i = 0$) do not participate in the aggregation rules; they serve a diagnostic role by providing additional context for human inspection but do not influence the prediction. The aggregation rules encode a key design principle grounded in the cost structure of scientific reasoning.

\begin{assumption}[Asymmetric Loss Structure]
\label{assump:loss}
The cost of prediction errors is asymmetric:
\begin{equation}
\label{eq:loss_ordering}
\ell_{\mathrm{fs}} \;>\; \ell_{\mathrm{fr}} \;>\; \ell_{\bot} = 0,
\end{equation}
where $\ell_{\mathrm{fs}}$ is the cost of false support (incorrectly affirming an unsupported claim), $\ell_{\mathrm{fr}}$ is the cost of false refutation (incorrectly rejecting a supported claim), and $\ell_{\bot}$ is the cost of abstention.
\end{assumption}

This loss structure reflects scientific practice, where affirming an unsupported conclusion (e.g., endorsing an ineffective treatment) is typically more consequential than withholding judgment. Under Assumption~\ref{assump:loss}, the decision rules are designed to be \emph{conservatively asymmetric}: a single contradicted critical condition is sufficient to invalidate a conclusion, while support requires all critical conditions to be affirmed.
\\ \\
\noindent
\textbf{Claim verification (SciFact).}
For claim verification, the aggregation rule is:
\begin{equation}
\label{eq:agg_scifact}
h(\mathbf{a}, \boldsymbol{\beta}) =
\begin{cases}
\texttt{REFUTES} & \text{if } \exists\, i\colon \beta_i\!=\!1 \;\wedge\; a_i = \textsc{con} \\[2pt]
\texttt{SUPPORTS} & \text{if } \forall\, i\colon \beta_i\!=\!1 \;\Rightarrow\; a_i = \textsc{sup} \\[2pt]
\texttt{NEI} & \text{otherwise}.
\end{cases}
\end{equation}

We note that the \texttt{REFUTES} and \texttt{SUPPORTS} cases are mutually exclusive: if every critical condition is supported, then no critical condition can be contradicted, since the audit status (Eq.~\ref{eq:audit_status}) assigns exactly one element of $\mathcal{A}$ to each condition. The \texttt{NEI} case captures all remaining configurations, including those where some critical conditions have missing evidence but none are contradicted.
\\ \\
\noindent
\textbf{Question answering (PubMedQA).}
For PubMedQA, audited conditions are mapped to answer labels as:
\begin{equation}
\label{eq:agg_pubmedqa}
h(\mathbf{a}, \boldsymbol{\beta}) =
\begin{cases}
\texttt{no} & \text{if } \exists\, i\colon \beta_i\!=\!1 \;\wedge\; a_i = \textsc{con} \\[2pt]
\texttt{yes} & \text{if } \exists\, i\colon \beta_i\!=\!1 \;\wedge\; a_i = \textsc{sup} \\
& \quad \wedge\;\; \nexists\, j\colon \beta_j\!=\!1 \;\wedge\; a_j = \textsc{con} \\[2pt]
\texttt{maybe} & \text{otherwise}.
\end{cases}
\end{equation}

Unlike claim verification, where support requires \emph{unanimity} among critical conditions (Eq.~\ref{eq:agg_scifact}), question answering adopts a weaker criterion: at least one supported critical condition suffices for an affirmative answer, provided no critical condition is contradicted. This reflects the structure of PubMedQA, where partial evidence from a biomedical abstract can be sufficient to indicate an affirmative answer, whereas claim verification demands comprehensive support for all necessary sub-claims.
\\ \\
\noindent
\textbf{Cross-stage consistency.}
In both rule sets, the refutation branch dominates: if any critical condition is contradicted, the conclusion is negative regardless of support elsewhere. This asymmetry is consistent with Assumption~\ref{assump:loss} and with the contradiction-priority ordering in the audit status assignment (Eq.~\ref{eq:audit_status}). The system thus enforces a coherent conservative bias across all pipeline stages---from evidence scoring through audit classification to final aggregation. In both tasks, the \texttt{NEI} and \texttt{maybe} labels serve as \emph{epistemic} outputs, indicating that the system recognizes evidence insufficiency prior to any abstention decision. Abstention (Section~\ref{sec:confidence}) provides an additional layer of control that operates over all label predictions, including these epistemic outputs, based on confidence.

\subsection{Confidence Estimation and Abstention}
\label{sec:confidence}

Abstention decisions are governed by an explicit confidence score derived from condition-level audit outcomes.

For each condition $c_i$, we compute a \emph{confidence margin}:
\begin{equation}
\label{eq:margin}
m_i = s_i^{\mathrm{ent}} - s_i^{\mathrm{con}},
\end{equation}
which captures the balance between supporting and contradicting evidence. Since $s_i^{\mathrm{ent}}, s_i^{\mathrm{con}} \in [0, 1]$, the margin satisfies $m_i \in [-1, 1]$. Only critical conditions contribute to instance-level confidence. The final confidence score is:
\begin{equation}
\label{eq:confidence}
\mathrm{conf}(x, \mathcal{E}) = \max_{c_i \in C_{\mathrm{crit}}} |m_i|.
\end{equation}
Since $|C_{\mathrm{crit}}| \geq 1$ by construction (Section~\ref{sec:decomposition}), the maximum is taken over a non-empty set and $\mathrm{conf}(x, \mathcal{E}) \in [0, 1]$.

\paragraph{Abstention rule.}
Abstention is applied via a threshold $\tau$:
\begin{equation}
\label{eq:abstention}
\hat{y}_\tau =
\begin{cases}
F(x, \mathcal{E}) & \text{if } \mathrm{conf}(x, \mathcal{E}) \geq \tau \\[2pt]
\bot & \text{otherwise}.
\end{cases}
\end{equation}

By varying $\tau$, the system traces a family of selective classifiers $\{(F, g_\tau)\}_{\tau \in [0,1]}$, enabling post-hoc control of the coverage--risk tradeoff without retraining. Additional discussion of the confidence design choices appears in Appendix~\ref{sec:appendix_confidence_rationale}.

\subsection{Risk--Coverage Evaluation}
\label{sec:rc}

To evaluate selective behavior, we compute risk--coverage curves by sweeping $\tau \in [0, 1]$. Let
$N_\tau = \sum_{i=1}^{N} \mathbf{1}[\mathrm{conf}(x_i, \mathcal{E}_i) \geq \tau]$
denote the number of non-abstained instances at threshold $\tau$. For $N_\tau > 0$, the empirical selective risk and empirical coverage are:
\begin{eqnarray}
\label{eq:emp_risk}
\hat{R}(\tau) &=& \frac{1}{N_\tau}\sum_{i=1}^{N} \mathbf{1}\!\big[F(x_i, \mathcal{E}_i) \neq y_i\big] \cdot \mathbf{1}\!\big[\mathrm{conf}(x_i, \mathcal{E}_i) \geq \tau\big] \\
\label{eq:emp_coverage}
\hat{\phi}(\tau) &=& \frac{N_\tau}{N}.
\end{eqnarray}

Sweeping $\tau$ yields a set of operating points $\{(\hat{\phi}(\tau), \hat{R}(\tau))\}$ that traces a risk--coverage curve.
\\ \\
\noindent
\textbf{Scalar summaries.}
We summarize the curve via the \emph{area under the risk--coverage curve} (AURC), computed by sorting operating points by coverage and applying trapezoidal integration over the discrete curve. Lower AURC indicates better selective performance across coverage levels. A formal definition is provided in Appendix~\ref{sec:appendix_aurc}. We additionally report \textbf{Risk@$\phi$}, the risk at fixed coverage levels $\phi \in \{0.8, 0.9\}$, capturing reliability at practically relevant operating points. 

\subsection{Theoretical Properties}
\label{sec:theory}

The proposed pipeline admits several formal properties that help explain its empirical reliability. In particular, under a mild rank-calibration assumption on the confidence score, selective risk is guaranteed to decrease monotonically as coverage is reduced. Additionally, the empirical selective risk concentrates around its population value at standard rates, implying that risk--coverage estimates are statistically reliable at moderate coverage levels. Finally, the deterministic aggregation rules can be interpreted as conservative approximations to Bayes-optimal decisions under asymmetric loss. Formal statements, proofs, and detailed discussion of assumptions are provided in Appendix~\ref{app:theory}.

\begin{figure}[t]
    \centering
    \includegraphics[width=\columnwidth]{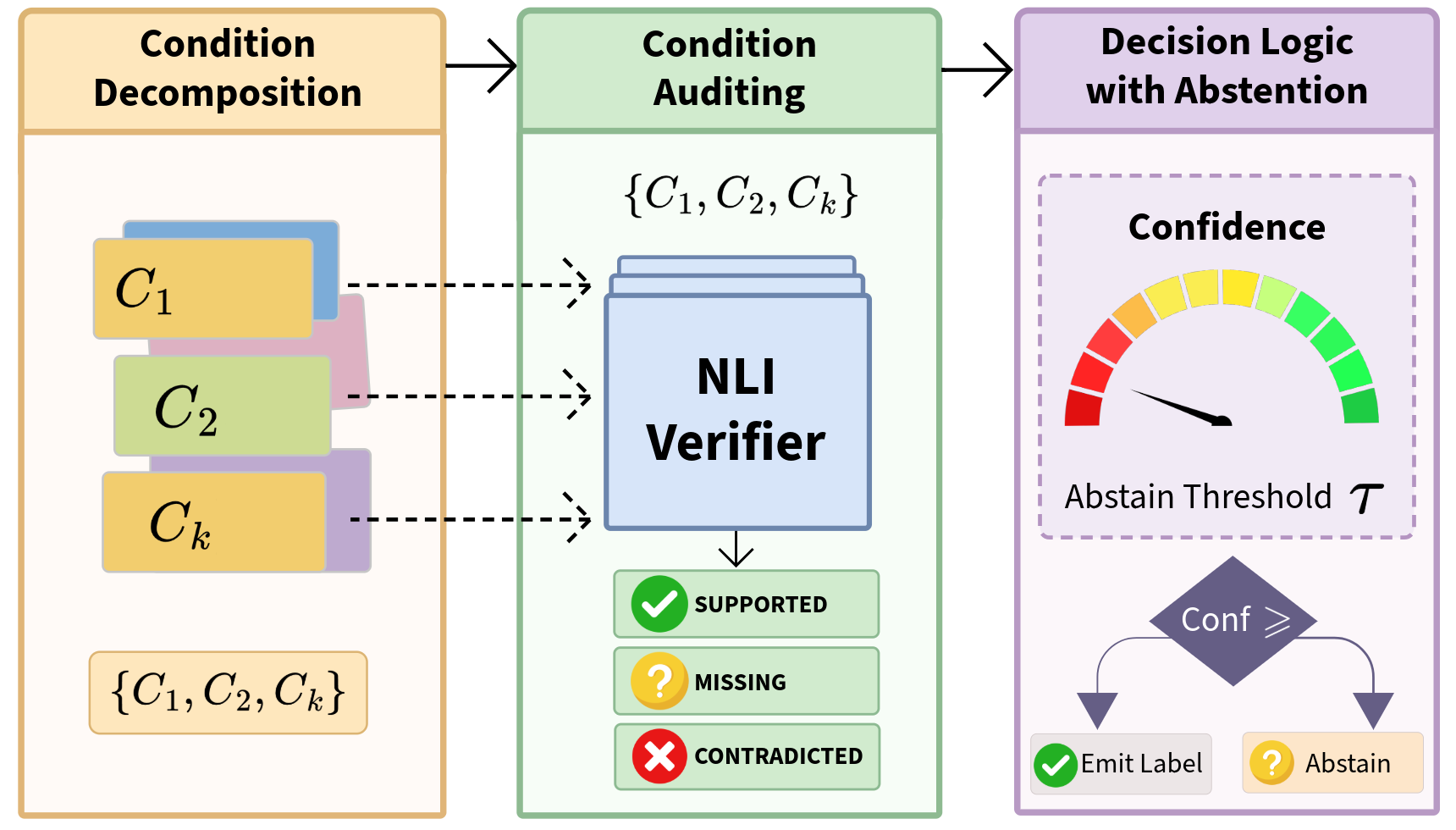}
    \caption{
    Overview of the abstention-aware scientific reasoning pipeline.
    An input claim or question is decomposed into a set of conditions.
    Each condition is independently evaluated against available evidence using the NLI verifier, producing condition-level judgments of support, contradiction, or missing evidence.
    These judgments are aggregated by a decision module that computes a confidence score and compares it against an abstention threshold $\tau$, emitting a final label only when the available evidence is sufficient.
    }
    \label{fig:system_overview}
\end{figure}

\section{Experiments}
\label{sec:exp}
\noindent
\textbf{Datasets.} 
We evaluate our approach on two scientific benchmarks that capture complementary reasoning settings. SciFact \cite{wadden-etal-2020-fact} focuses on scientific claim verification, where each input is a declarative claim paired with curated evidence annotated as supporting, refuting, or insufficient. PubMedQA \cite{jin-etal-2019-pubmedqa} evaluates biomedical question answering, requiring models to answer yes, no, or maybe questions based on abstracts from the biomedical literature. Together, these datasets span both verification and question answering tasks, and represent domains where incorrect conclusions can be particularly costly. For each dataset, we evaluate performance under both unconditional prediction and selective prediction with abstention.
\\ \\
\noindent
\textbf{Models.}
We evaluate six language models spanning multiple architectural and deployment regimes. Encoder--decoder models are represented by Flan-T5-large\cite{chung2022scalinginstructionfinetunedlanguagemodels}, accessed via HuggingFace Transformers and run locally. Open-weight chat models include Llama\cite{grattafiori2024llama3herdmodels} (llama-3.3-70b-instruct), Mistral\cite{jiang2023mistral7b}, and DeepSeek\cite{deepseekai2025deepseekr1incentivizingreasoningcapability}, accessed through hosted inference APIs and treated as black-box generators. We also evaluate proprietary models accessed through external APIs, namely gpt-4o-mini\cite{openai2024gpt4ocard} and gpt-5.2\cite{singh2025openaigpt5card}. Openrouter was used for API access to the models. To ensure consistent and centralized access to hosted language models, all API-based models are accessed through OpenRouter\footnote{https://openrouter.ai/}, which provides a unified endpoint for querying multiple providers under identical request and logging semantics. All models are used exclusively for condition decomposition and answer generation; evidence auditing is performed by a fixed NLI model across all experiments. Specifically, we use a DeBERTa-based cross-encoder\footnote{https://huggingface.co/cross-encoder/nli-deberta-v3-base} trained for NLI to score entailment, contradiction, and neutrality between conditions and evidence sentences. This design ensures that performance differences arise from the generative model and the decision structure, rather than from changes in the audit component.
\\ \\
\noindent
\textbf{Experimental Setup.}
For each dataset, we fix the evidence inputs, decision thresholds, and auditing model across all runs. Condition decomposition is performed using the specified generative model unless explicitly disabled in ablation settings. For SciFact, gold evidence sentences provided by the dataset are used. For PubMedQA, evidence consists of all sentences in the associated abstract, truncated to a fixed maximum length. Predictions are generated deterministically with zero temperature decoding. Abstention decisions are derived from the final confidence score produced by the decision module, and risk--coverage curves are computed by varying the abstention threshold post hoc. All runs are executed with fixed random seeds, and outputs including predictions, confidence scores, and audit traces are logged to disk for reproducibility. Configuration files and scripts used to generate all results are included in the shared code repository.

\section{Results}
\label{sec:results}

We evaluate abstention-aware scientific reasoning on two benchmarks that capture complementary scientific settings: SciFact, which focuses on claim verification with curated evidence, and PubMedQA, which evaluates question answering grounded in biomedical abstracts. For each dataset, we report both standard classification metrics and selective prediction metrics to characterize model behavior under uncertainty. In addition to accuracy and macro-F1, which measure unconditional performance when a model must answer every query, we report the area under the risk–coverage curve (AURC) and risk at fixed coverage levels. Risk is defined as the empirical error rate among non-abstained predictions, while coverage denotes the fraction of instances on which a model chooses to answer. Lower AURC and lower risk at a given coverage indicate safer behavior, as the model concentrates its errors among abstained instances rather than committed predictions. Together, these metrics allow us to distinguish between models that achieve similar accuracy but differ substantially in how reliably they decide when available evidence is sufficient to justify an answer.

\begin{table*}[t]
\centering
\small
\setlength{\tabcolsep}{3pt}
\renewcommand{\arraystretch}{1.05}
\begin{tabular}{l l c c c c c}
\toprule
\textbf{Task} & \textbf{Model} & Acc. $\uparrow$ & F1 $\uparrow$ & AURC $\downarrow$ & R@0.8 $\downarrow$ & R@0.9 $\downarrow$ \\
\midrule
\multirow{6}{*}{PubMedQA}
& flan-t5-large   & 0.56  & 0.374 & 0.389 & 0.463 & 0.451 \\
& gpt-4o-mini     & 0.485 & 0.355 & 0.411 & 0.560 & 0.542 \\
& deepseek-chat   & 0.44  & 0.312 & 0.444 & 0.590 & 0.575 \\
& llama-3.3-70b   & 0.40  & 0.300 & 0.429 & 0.619 & 0.633 \\
& mistral-large   & 0.375 & 0.275 & 0.443 & 0.663 & 0.651 \\
& gpt-5.2-chat    & 0.35  & 0.254 & 0.464 & 0.656 & 0.675 \\
\midrule
\multirow{6}{*}{SciFact}
& mistral-large   & 0.405 & 0.192 & 0.221 & 0.506 & 0.549 \\
& deepseek-chat   & 0.40  & 0.191 & 0.224 & 0.506 & 0.561 \\
& llama-3.3-70b   & 0.40  & 0.190 & 0.225 & 0.506 & 0.556 \\
& flan-t5-large   & 0.395 & 0.189 & 0.228 & 0.496 & 0.551 \\
& gpt-4o-mini     & 0.395 & 0.189 & 0.223 & 0.506 & 0.561 \\
& gpt-5.2-chat    & 0.395 & 0.189 & 0.226 & 0.506 & 0.559 \\
\bottomrule
\end{tabular}
\vspace{-0.4em}
\caption{Summary of unconditional accuracy metrics and selective prediction metrics across tasks and models. Risk@0.8 and Risk@0.9 denote risk at 0.8 and 0.9 coverage, respectively.}
\label{tab:main_results}
\vspace{-0.6em}
\end{table*}

\noindent
\textbf{Limited Separation Under Accuracy Metrics.}
Across both SciFact and PubMedQA, we observe that raw accuracy and macro-F1 provide only limited separation between models, as summarized in Table~\ref{tab:main_results}. Despite substantial differences in architecture, scale, and training methodology, most models cluster within a narrow performance range under unconditional evaluation. On SciFact, accuracies concentrate around the high 0.3 to low 0.4 range, with macro-F1 scores differing only marginally across all evaluated models. Even large-scale chat models and proprietary APIs do not substantially outperform smaller or instruction-tuned alternatives when forced to answer every claim. PubMedQA exhibits slightly higher variance in accuracy, yet the overall pattern remains similar: no single model consistently dominates across both datasets under standard metrics. These results indicate that, for scientific verification and biomedical question answering, absolute accuracy alone is a weak indicator of reliability. Improvements in model scale or architectural sophistication do not translate into proportional gains in correctness when models are required to answer every query, particularly in settings where evidence is incomplete or ambiguous.

\begin{figure}[t]
    \centering
    \begin{subfigure}[t]{0.86\linewidth}
        \centering
        \includegraphics[width=\linewidth]{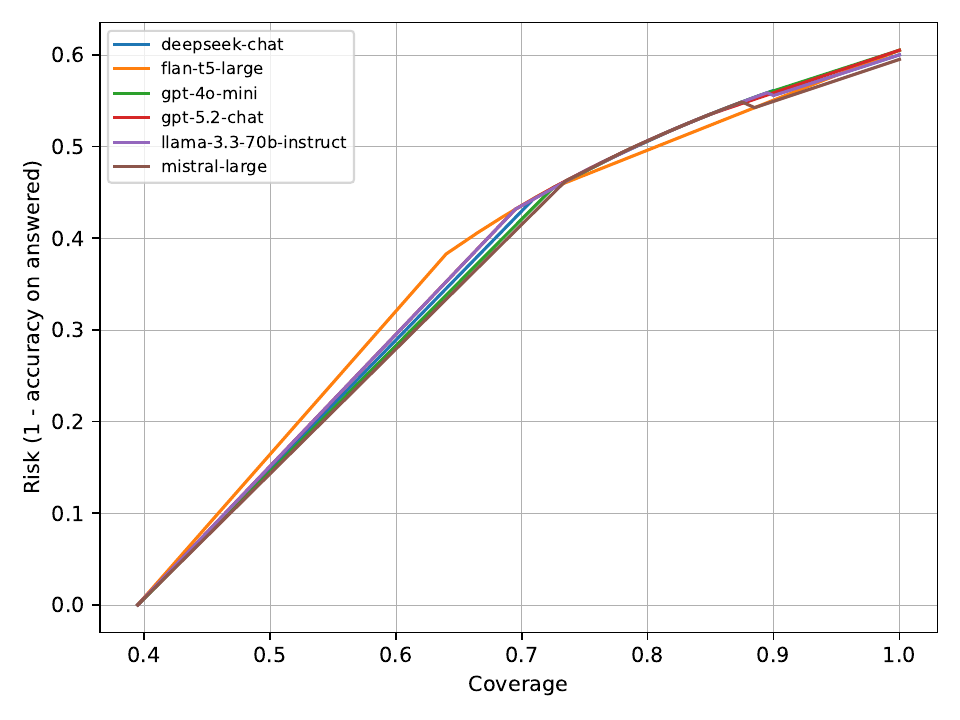}
    \end{subfigure}

    \begin{subfigure}[t]{0.86\linewidth}
        \centering
        \includegraphics[width=\linewidth]{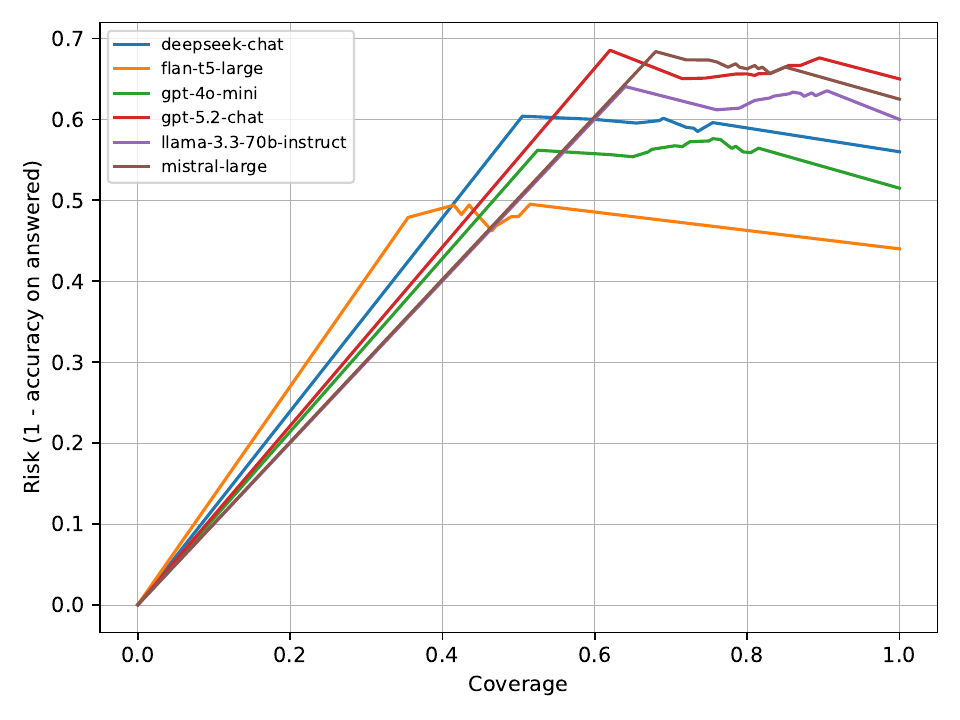}
    \end{subfigure}

    \vspace{-0.3em}
    \caption{\textbf{Risk--coverage curves} for (top) SciFact and (bottom) PubMedQA. Lower curves indicate better reliability under selective prediction.}
    \label{fig:rc_main}
\end{figure}

\noindent
\textbf{Risk--Coverage Reveals Substantial Reliability Gains.}
In contrast, selective evaluation showcases systematic differences that are not visible through accuracy metrics alone. Risk--coverage curves in Figure~\ref{def:risk_coverage}, together with quantitative summaries in Table~\ref{tab:main_results}, show that all models exhibit a pronounced tradeoff between coverage and risk. At full coverage, risk remains high across both datasets, reflecting frequent incorrect or weakly supported answers. As coverage is reduced through confidence-based abstention, risk declines sharply. For both SciFact and PubMedQA, abstaining reduces error among answered predictions by a substantial margin, often halving risk relative to full coverage. These trends are consistent across all evaluated models. Differences between models manifest primarily as shifts along similar risk--coverage curves rather than fundamentally different tradeoffs. This suggests that abstention behavior, rather than architectural differences, is the dominant factor governing reliability under selective decision making. While we do not explicitly report calibration metrics such as expected calibration error, risk--coverage curves implicitly characterize calibration among accepted predictions. As coverage decreases, risk consistently declines, indicating that confidence scores correlate with correctness. We also note that confidence is derived from evidence-level signals rather than model logits, making it inherently sensitive to evidence completeness. Investigating robustness under systematic evidence removal is a promising extension.
\\ \\
\noindent
\textbf{Instruction Tuning and Conservative Decision Making.}
While raw accuracy remains tightly clustered, selective risk metrics reveal a consistent pattern: instruction-tuned encoder--decoder models such as FLAN-T5 frequently achieve the lowest risk at comparable coverage levels. As shown in Table~\ref{tab:main_results}, FLAN-T5 attains among the lowest AURC values and reduced risk at both 80\% and 90\% coverage, despite exhibiting accuracy comparable to or lower than larger chat-oriented models. We attribute this behavior to more conservative response patterns and tighter adherence to task instructions. In practice, FLAN-T5 produces fewer speculative condition decompositions and more readily abstains when evidence is insufficient. This aligns well with abstention-aware evaluation, where avoiding unsupported answers is prioritized over maximizing coverage. As a result, FLAN-T5 demonstrates stronger reliability under selective prediction even without improvements in unconditional accuracy. Conversely, larger chat-style models often maintain higher coverage at the cost of increased risk, indicating a tendency to answer confidently even in uncertain settings. This tradeoff is not penalized under standard accuracy metrics but becomes visible under risk--coverage analysis.
\\ \\
\noindent
\textbf{Abstention Dominates Model Choice.}
A central finding across both datasets is that abstention strategy has a greater impact on reliability than model selection. When operating at comparable coverage levels, weaker and stronger models often achieve similar risk. In several cases, a smaller or instruction-tuned model operating at reduced coverage attains lower risk than a larger model forced to answer all queries. This effect holds consistently across SciFact and PubMedQA and across all evaluated model families. Once coverage is controlled, performance differences between models narrow substantially. These results indicate that scientific reliability depends less on maximizing answer frequency and more on recognizing when available evidence is insufficient to justify a conclusion.
\\ \\
\noindent
\textbf{Differences Between Verification and Question Answering.}
Although the overall trends are consistent, SciFact and PubMedQA exhibit different operating regimes. SciFact shows uniformly high risk at full coverage, reflecting the difficulty of claim verification even when curated evidence is available. PubMedQA allows slightly higher safe coverage before risk increases sharply, likely due to richer contextual information in biomedical abstracts. Nevertheless, abstention remains critical in both settings. Even in PubMedQA, forcing models to answer all questions leads to a rapid increase in error, underscoring the importance of selective prediction in scientific domains.
\\ \\ 
\noindent
\textbf{Implications for Scientific Evaluation.}
Taken together, these results suggest that the primary challenge in scientific reasoning is not selecting a single best-performing model, but determining when an answer is justified by available evidence. Traditional accuracy-focused evaluation obscures this distinction by conflating confident errors with informed decisions. Abstention-aware evaluation provides a more faithful characterization of scientific reliability. By explicitly modeling the decision to answer or abstain, it highlights the conditions under which models can be trusted and exposes failure modes that are invisible under standard metrics. These findings motivate the use of risk--coverage analysis as a standard component of evaluation for scientific reasoning systems, particularly in settings where incorrect answers carry high cost.

\section{Ablations}
\label{sec:ablations}
To better understand the contribution of individual components in our verification pipeline, we conduct a set of targeted ablation experiments. The goal of these experiments is not to maximize overall performance, but to isolate the effect of claim decomposition, evidence auditing, and selective abstention on reliability. To limit experimental complexity and improve interpretability, all ablations are conducted using two representative language models: one open-weight model (Mistral Large) and one proprietary model (GPT-5.2). Unless otherwise noted, all other settings are kept identical to the full pipeline. Since the no-audit and no-abstain settings bypass model-dependent reasoning stages, outcomes collapse to deterministic decision rules shared across models, resulting in identical metrics across architectures. This behavior is expected and highlights which components are responsible for meaningful variation in scientific reliability. Some ablation settings, particularly \texttt{no-audit}, yield degenerate metrics such as 100\% accuracy or zero risk. These results do not indicate perfect reasoning performance. Rather, they arise because disabling auditing collapses the decision rule to a trivial default that mirrors the majority or null label, while confidence-based abstention suppresses erroneous predictions. We therefore interpret these ablations diagnostically rather than comparatively. Their purpose is to demonstrate that each component of the pipeline is necessary to avoid degenerate behavior, not to serve as competitive baselines.\\ \\
\noindent
\textbf{Ablation 1: No Decomposition.}
In this setting, the claim decomposition stage is removed. Each input is treated as a single atomic condition, which is passed directly to the downstream auditing and decision modules. This ablation tests whether explicitly decomposing scientific claims into minimal conditions contributes meaningfully to verification reliability, or whether similar behavior can be achieved by auditing the claim as a whole.  In the \emph{no-decompose} setting, each claim or question is treated as a single unit rather than being decomposed into minimal conditions. Interestingly, as shown in the Table \ref{tab:ablations_singlecol}, this setting can yield comparable or even lower selective risk than the full pipeline at certain coverage levels. This behavior is also visible when comparing AURC and Risk@0.8 values against those in Table~\ref{tab:main_results}. This apparent improvement does not contradict the role of decomposition. Instead, it reflects a more conservative decision process: without decomposition, the system audits a single coarse condition, which often leads to higher uncertainty scores and earlier abstention. As a result, fewer borderline cases are answered, artificially lowering risk under selective evaluation. While this behavior can be beneficial from a risk minimization perspective, it comes at the cost of expressiveness and diagnostic power. The system can no longer identify which specific sub-claims are unsupported, nor can it distinguish partial support from complete failure. Decomposition therefore serves a structural role, enabling fine-grained attribution and principled auditing, even if it sometimes exposes additional failure modes that increase measured risk.\\ \\
\noindent
\textbf{Ablation 2: No Audit.}
This ablation disables the evidence auditing stage entirely. The model still generates conditions, but no NLI checks are performed against the provided evidence. As a result, the system defaults to a non-committal decision with minimal confidence. This setting serves as a lower-bound baseline, illustrating the necessity of explicit evidence verification rather than relying solely on generation or heuristics. Table~\ref{tab:ablations_singlecol} shows the effect of removing evidence auditing entirely. In this setting, both models achieve near-perfect or perfect accuracy on SciFact, far exceeding the values reported in Table~\ref{tab:main_results}. However, this improvement is illusory. Without auditing, predictions are no longer conditioned on evidence, and decisions collapse to trivial or dataset-biased heuristics. Selective metrics make this failure explicit. Risk and AURC collapse to degenerate values, indicating that confidence no longer reflects correctness. Compared to the full pipeline in Table~\ref{tab:main_results}, the no-audit ablation demonstrates that high accuracy alone is insufficient and can be actively misleading. Evidence auditing is therefore essential for grounding predictions and preventing spurious reliability gains.

\begin{table*}[t]
\centering
\normalsize
\setlength{\tabcolsep}{2.5pt}
\renewcommand{\arraystretch}{1.0}
\begin{tabular}{l l l c c c c}
\toprule
\textbf{Ablation} & \textbf{Model} & \textbf{Task} & Acc $\uparrow$ & F1 $\uparrow$ & AURC $\downarrow$ & R@0.8 $\downarrow$ \\
\midrule
\multirow{4}{*}{\textbf{A1}}
& \multirow{2}{*}{GPT-5.2} 
& PubMedQA & 0.55  & 0.313 & 0.444 & 0.472 \\
& 
& SciFact  & 0.395 & 0.189 & 0.227 & 0.502 \\
& \multirow{2}{*}{Mistral} 
& PubMedQA & 0.55  & 0.313 & 0.444 & 0.472 \\
& 
& SciFact  & 0.395 & 0.189 & 0.227 & 0.502 \\
\midrule
\multirow{4}{*}{\textbf{A2}}
& \multirow{2}{*}{GPT-5.2} 
& PubMedQA & 0.175 & 0.099 & 0.413 & 0.660 \\
& 
& SciFact  & 1.000 & 0.333 & 0.000 & 0.000 \\
& \multirow{2}{*}{Mistral} 
& PubMedQA & 0.175 & 0.099 & 0.413 & 0.660 \\
& 
& SciFact  & 1.000 & 0.333 & 0.000 & 0.000 \\
\midrule
\multirow{4}{*}{\textbf{A3}}
& \multirow{2}{*}{GPT-5.2} 
& PubMedQA & 0.34  & 0.243 & 0.471 & 0.717 \\
& 
& SciFact  & 0.000 & 0.000 & 0.843 & 1.000 \\
& \multirow{2}{*}{Mistral} 
& PubMedQA & 0.395 & 0.289 & 0.466 & 0.663 \\
& 
& SciFact  & 0.000 & 0.000 & 0.833 & 1.000 \\
\bottomrule
\end{tabular}
\vspace{-0.3em}
\caption{\textbf{Ablation results.}
A1: no decomposition; A2: no auditing; A3: no abstention.
Acc/F1 report unconditional performance, while AURC and R@0.8 measure selective reliability (lower is better).}
\label{tab:ablations_singlecol}
\vspace{-0.6em}
\end{table*}

\noindent
\textbf{Ablation 3: No Abstention.}
In the final ablation, the selective abstention mechanism is removed. The system is forced to produce a definitive answer for every input, regardless of confidence or evidence support. This experiment isolates the role of abstention in controlling error and directly quantifies the trade-off between coverage and risk when abstention is disallowed. Table~\ref{tab:ablations_singlecol} reports results when abstention is disabled and the system is forced to answer every query. Across both datasets and models, this ablation consistently produces the highest risk and worst selective metrics. Risk@0.8 and Risk@0.9 approach one in multiple settings, indicating that most high-confidence answers are incorrect. In contrast, the full pipeline in Table~\ref{tab:main_results} shows substantial reductions in risk at comparable coverage levels. This confirms that abstention is the primary mechanism by which the system controls error. While decomposition and auditing shape how evidence is evaluated, abstention governs whether an answer should be produced at all. Its removal fundamentally breaks selective prediction and exposes the brittleness of unconditional answering in scientific domains.

 \section{Limitations}
 \label{sec:limitations}
While the proposed framework demonstrates consistent reliability gains across scientific verification tasks, several limitations are worth noting.
\\
\noindent
\textbf{NLI model reliance.} This approach relies on NLI models to audit conditions against evidence. As a result, errors or biases in the NLI model can propagate to downstream decisions. This limitation is inherent to any verification pipeline that separates reasoning from evidence evaluation. However, our experiments intentionally fix the NLI component across all conditions and models, ensuring that observed differences arise from decision structure rather than from changes in the verifier itself. Moreover, the framework is modular and allows stronger or domain-specific NLI models to be substituted without altering the overall methodology.
\\
\noindent
\textbf{Evaluation Scope.} Our evaluation focuses on two curated scientific benchmarks, SciFact and PubMedQA, which emphasize textual evidence and abstract-level reasoning. While these datasets cover both claim verification and question answering, they do not capture all forms of scientific reasoning, such as numerical simulation, experimental design, or multi-document synthesis at scale. However, the proposed pipeline is agnostic to the source of evidence and can be extended to other scientific modalities as suitable datasets become available.
\\
\noindent
\textbf{Evaluation Metrics.} Our analysis prioritizes risk–coverage behavior over raw accuracy, which may differ from conventional evaluation practices. While this choice reflects real-world scientific use cases where incorrect answers are costly, it may limit direct comparison with prior work that reports only accuracy-based metrics. To address this, we report both standard classification metrics and risk–coverage curves, enabling readers to interpret results under both evaluation paradigms.

\section{Conclusion and Future Work}
\label{sec:conclusion}

We introduced an abstention-aware framework that decomposes scientific inputs into minimal conditions, audits each condition against evidence, and explicitly determines when available evidence is sufficient to support a conclusion. Across SciFact and PubMedQA, we show that this structure consistently reduces error at moderate coverage levels, even when gains in raw accuracy are modest.  Our results reveal that differences between language models are often secondary to the impact of principled abstention. Reliability improves primarily by controlling when models answer, not by forcing a definitive prediction in all cases. Risk--coverage analysis exposes this trade-off more clearly than accuracy alone, particularly in high-stakes scientific settings.
Beyond empirical findings, this work frames scientific reasoning as a selective inference problem and separates reasoning, evidence auditing, and decision making into modular components. This design aligns with scientific practice, where uncertainty is expected in the absence of sufficient evidence. Future work includes extending the framework to multimodal data, incorporating domain-specific verifiers, and developing benchmarks that explicitly reward calibrated abstention.




\appendix

\section{Additional Methodological Details and Theory}
\label{sec:appendix}

\subsection{Factorization View and Audit Independence}
\label{sec:appendix_factorization}

From a modeling perspective, decomposition induces a structured factorization of the reasoning task. Rather than modeling $P(y \mid x, \mathcal{E})$ as a monolithic function, we decompose reasoning through condition-level assessments. By the law of total probability, marginalizing over audit outcome vectors:
\begin{equation}
\label{eq:factorization}
P(y \mid x, \mathcal{E}) \;=\; \sum_{\mathbf{a} \,\in\, \mathcal{A}^k} P\!\big(y \mid \mathbf{a},\, \boldsymbol{\beta},\, x,\, \mathcal{E}\big) \;\cdot\; P\!\big(\mathbf{a} \mid c_1, \dots, c_k,\, \mathcal{E}\big),
\end{equation}
where $\mathbf{a} = (a_1, \dots, a_k)$ and $\boldsymbol{\beta} = (\beta_1, \dots, \beta_k)$. Note that $P(\mathbf{a} \mid c_1, \dots, c_k, \mathcal{E})$ does not condition on $x$ directly; audit outcomes depend on $x$ only through the decomposed conditions, since the NLI model receives only $(c_i, e_j)$ pairs. Eq.~\ref{eq:factorization} is exact and introduces no approximation.

The pipeline introduces two approximations to this expression. First, the joint audit distribution is factored under conditional independence (Assumption~\ref{assump:cond_indep}). Second, the stochastic label model $P(y \mid \mathbf{a}, \boldsymbol{\beta}, x, \mathcal{E})$ is replaced by the deterministic aggregation function $h$ (Section~\ref{sec:decision}), combined with hard-threshold discretization of audit outcomes (Section~\ref{sec:audit}). Together these yield the working approximation:
\begin{equation}
\label{eq:approx_factorization}
P(y \mid x, \mathcal{E}) \;\approx\; \sum_{\mathbf{a} \,\in\, \mathcal{A}^k} \mathbf{1}\!\big[h(\mathbf{a}, \boldsymbol{\beta}) = y\big] \;\prod_{i=1}^{k} P(a_i \mid c_i, \mathcal{E}).
\end{equation}

Eq.~\ref{eq:approx_factorization} makes explicit that two modeling choices are combined: conditional independence (the product) and deterministic aggregation (the indicator replacing the stochastic $P(y \mid \mathbf{a}, \boldsymbol{\beta}, x, \mathcal{E})$). The independence assumption is formalized below; the deterministic aggregation is a design choice justified by its interpretability and consistency with the asymmetric loss structure (Section~\ref{sec:decision}).

\begin{assumption}[Conditional Independence of Audits]
\label{assump:cond_indep}
Given evidence set $\mathcal{E}$ and a fixed decomposition $C = \{c_1, \dots, c_k\}$, condition audit outcomes are approximately independent:
\begin{equation}
\label{eq:cond_indep}
P(a_1, \dots, a_k \mid c_1, \dots, c_k, \mathcal{E}) \;\approx\; \prod_{i=1}^{k} P(a_i \mid c_i, \mathcal{E}).
\end{equation}
\end{assumption}

This assumption holds when conditions target semantically distinct aspects of the input and are evaluated against different evidence sentences. It may be violated when multiple conditions depend on the same evidence sentence (e.g., overlapping factual claims drawn from a single source), in which case audit outcomes become positively correlated. In the worst case, such correlations can cause the pipeline to double-count evidence---treating a single supporting sentence as independent support for multiple conditions---which inflates confidence. However, the max-aggregation in the confidence score (Section~\ref{sec:confidence}) partially mitigates this effect, since it depends on a single condition's margin rather than accumulating signals across conditions.

Empirically, the ablation removing decomposition (Section~\ref{sec:ablations}) provides an indirect test: without decomposition, the independence assumption is trivially satisfied ($k=1$), and the resulting changes in risk--coverage behavior are moderate, suggesting the pipeline is not critically sensitive to this assumption. We note that the theoretical analysis in subsequent sections conditions on a fixed decomposition $C$ and criticality assignment $\boldsymbol{\beta}$. Since decomposition is performed by a generative model, variability across runs could in principle affect downstream guarantees. However, we observe empirically that the number and semantic roles of critical conditions are stable across runs (Section~\ref{sec:ablations}), suggesting this conditioning is practically reasonable. Across datasets, inputs typically decompose into 2--4 conditions, with approximately 60--70\% marked as critical.

\subsection{Confidence Design Rationale and Alternatives}
\label{sec:appendix_confidence_rationale}

\paragraph{Design rationale.}
The max-absolute-margin aggregation encodes a \emph{decisive evidence} principle: confidence is governed by the single strongest evidential signal among critical conditions, whether supporting or contradicting. This reflects a design choice that a single well-supported critical condition provides more decision-relevant information than the average signal across all conditions.

We note that the margin $m_i$ does not distinguish between \emph{absence} of evidence ($s_i^{\mathrm{ent}} \approx s_i^{\mathrm{con}} \approx 0$) and \emph{conflicting} evidence ($s_i^{\mathrm{ent}} \approx s_i^{\mathrm{con}} \gg 0$); both yield $|m_i| \approx 0$ and thus low confidence. While both appropriately trigger caution, they represent distinct epistemic states---ignorance versus disagreement. Disentangling these cases, for example through separate uncertainty types, is a direction for future work.

\paragraph{Alternative aggregation strategies.}
Alternative confidence aggregations encode different inductive biases:
\begin{itemize}
    \item \textbf{Min-margin}: $\mathrm{conf}_{\min}(x) = \min_{c_i \in C_{\mathrm{crit}}} |m_i|$. This implements a \emph{weakest-link} principle, where confidence is limited by the least-supported critical condition, typically reducing coverage while improving reliability.
    \item \textbf{Mean-margin}: $\mathrm{conf}_{\mathrm{avg}}(x) = \frac{1}{|C_{\mathrm{crit}}|} \sum_{c_i \in C_{\mathrm{crit}}} |m_i|$. This smooths over individual conditions but can dilute decisive signals and obscure the contribution of any single critical condition.
\end{itemize}

We adopt max-margin as the default because it is most sensitive to the presence of any strong evidence. In particular, under our asymmetric loss setting, a single clear contradictory signal should yield high confidence in a refutation even if other conditions remain ambiguous.

\subsection{Formal AURC Definition via Generalized Inverse}
\label{sec:appendix_aurc}

Let $r\colon (0,1] \to [0,1]$ denote the function mapping each coverage level $c$ to the selective risk achieved at that coverage:
\begin{equation}
\label{eq:rc_function}
r(c) = R\!\big(F,\, g_{\tau(c)}\big), \quad \text{where}\quad \tau(c) = \inf\!\big\{\tau \geq 0 : \phi(\tau) \leq c\big\}
\end{equation}
is the generalized inverse of the non-increasing coverage function $\phi(\tau)$. The AURC is then:
\begin{equation}
\label{eq:aurc}
\mathrm{AURC} = \int_0^1 r(c)\, dc.
\end{equation}

Since $r(c) \in [0, 1]$ for all $c \in (0, 1]$ and $r$ is bounded, the integral is well-defined even though $r$ may be undefined at $c = 0$ (where no instances are selected); the single-point boundary has Lebesgue measure zero and does not affect the integral. In practice, AURC is approximated via trapezoidal integration over the discrete set of operating points induced by sweeping $\tau$.

\subsection{Computational Complexity}
\label{sec:appendix_complexity}

For a single input with $k$ decomposed conditions and $n$ evidence sentences, the auditing stage requires $k \cdot n$ NLI forward passes. In practice, $k \in \{2, 4\}$ and $n$ is bounded by the abstract length, yielding $O(kn)$ inference calls per instance. Since the NLI model is a fixed cross-encoder applied independently to each pair, these calls are trivially parallelizable, and the total pipeline cost is dominated by the generative model used for condition decomposition.

\section{Theoretical Analysis}
\label{app:theory}

This appendix provides a formal analysis of the abstention-aware verification pipeline. All results condition on a fixed condition decomposition $C = \{c_1,\dots,c_k\}$ and criticality assignment $\boldsymbol{\beta}$, as produced by the decomposition stage. The goal of this section is to explain why the proposed confidence-based abstention mechanism yields reliable selective behavior, and how the aggregation rules align with asymmetric decision costs in scientific reasoning.

\subsection{Monotonicity of Selective Risk}
\label{app:monotonicity}

We begin by formalizing a mild ordering condition on the confidence score under which selective risk improves monotonically as coverage is reduced.

\begin{definition}[Rank-Calibration]
A confidence score $\mathrm{conf}\colon \mathcal{X} \times \mathbb{E} \to [0,1]$ is \emph{rank-calibrated} with respect to a prediction function $F$, loss $\ell$, and data distribution $\mathcal{D}$ if for all $t_1 < t_2$ with $\phi(t_1) > \phi(t_2) > 0$,

\begin{equation}
\label{eq:rank_calibration}
\begin{aligned}
\mathbb{E}\!\big[\ell(F(x,\mathcal{E}),y)\mid \mathrm{conf}(x,\mathcal{E})\in[t_1,t_2)\big]
\;\ge\;& \\
\mathbb{E}\!\big[\ell(F(x,\mathcal{E}),y)\mid \mathrm{conf}(x,\mathcal{E})\ge t_2\big].&
\end{aligned}
\end{equation}

\end{definition}

Rank-calibration requires only that lower-confidence predictions incur equal or greater expected loss than higher-confidence predictions. It is strictly weaker than probabilistic calibration and depends only on the ordering induced by the confidence score.

\begin{proposition}[Monotonicity of Selective Risk]
If $\mathrm{conf}$ is rank-calibrated with respect to $F$, $\ell$, and $\mathcal{D}$, then for any $\tau_1 < \tau_2$ with $\phi(\tau_2) > 0$,
\begin{equation}
R(F,g_{\tau_2}) \le R(F,g_{\tau_1}).
\end{equation}
\end{proposition}

\begin{proof}
Since $\phi(\tau)$ is non-increasing in $\tau$, $\phi(\tau_1) \ge \phi(\tau_2) > 0$. Partition the accepted set at threshold $\tau_1$ into
\[
A_{\ge} = \{\mathrm{conf}(x,\mathcal{E}) \ge \tau_2\}, \quad
A_{\mathrm{mid}} = \{\mathrm{conf}(x,\mathcal{E}) \in [\tau_1,\tau_2)\}.
\]
By the law of total expectation,
\begin{equation}
R(F,g_{\tau_1})
=
\frac{\phi(\tau_2)}{\phi(\tau_1)} R(F,g_{\tau_2})
+
\frac{\phi(\tau_1)-\phi(\tau_2)}{\phi(\tau_1)} R_{\mathrm{mid}},
\end{equation}
where
\[
R_{\mathrm{mid}} =
\mathbb{E}\!\big[\ell(F(x,\mathcal{E}),y)\mid \mathrm{conf}(x,\mathcal{E}) \in [\tau_1,\tau_2)\big].
\]
By rank-calibration, $R_{\mathrm{mid}} \ge R(F,g_{\tau_2})$. Since the above expression is a convex combination, the inequality follows.
\end{proof}

This result guarantees that tightening the abstention threshold cannot increase risk among accepted predictions, explaining the monotone risk--coverage curves observed empirically.

\subsection{Finite-Sample Concentration of Selective Risk}
\label{app:concentration}

We now quantify the reliability of empirical selective risk estimates computed from finite datasets.

\begin{proposition}[Finite-Sample Concentration]
Let $\{(x_i,\mathcal{E}_i,y_i)\}_{i=1}^{N}$ be drawn i.i.d.\ from $\mathcal{D}$. Under 0--1 loss, for any fixed threshold $\tau$ and $\epsilon > 0$, conditional on $N_\tau = n > 0$,
\begin{equation}
P\!\left(|\hat{R}(\tau) - R(F,g_\tau)| > \epsilon \,\middle|\, N_\tau=n\right)
\le
2\exp(-2n\epsilon^2).
\end{equation}
\end{proposition}

\begin{proof}
Condition on $N_\tau = n$. The selected instances are i.i.d.\ under the conditional distribution $\mathcal{D}_\tau = (\mathcal{D} \mid \mathrm{conf}(x,\mathcal{E}) \ge \tau)$. Define
\[
Z_j = \mathbf{1}[F(x_j,\mathcal{E}_j) \neq y_j], \quad j=1,\dots,n.
\]
Each $Z_j$ is Bernoulli with mean $R(F,g_\tau)$ and bounded in $[0,1]$. Since
\[
\hat{R}(\tau) = \frac{1}{n} \sum_{j=1}^{n} Z_j,
\]
Hoeffding’s inequality yields the stated bound.
\end{proof}

This result formalizes the intuition that selective risk estimates become unstable at very low coverage, motivating the focus on moderate-coverage regimes in evaluation.

\subsection{Decision-Theoretic Justification of Asymmetric Aggregation}
\label{app:decision_theory}

We now connect the deterministic aggregation rules to Bayes-optimal decision-making under asymmetric loss.

\begin{proposition}[Bayes-Optimal Decision Threshold]
Consider binary labels $y_+ = \texttt{SUPPORTS}$ and $y_- = \texttt{REFUTES}$ with asymmetric loss $\ell_{\mathrm{fs}} > \ell_{\mathrm{fr}} > 0$. Given observations $\mathbf{o}$, the Bayes-optimal decision predicts $y_+$ if and only if
\begin{equation}
\frac{P(y_+ \mid \mathbf{o})}{P(y_- \mid \mathbf{o})}
>
\frac{\ell_{\mathrm{fs}}}{\ell_{\mathrm{fr}}}.
\end{equation}
\end{proposition}

\begin{proof}
The expected losses are
\[
\mathbb{E}[\ell(y_+,y)\mid\mathbf{o}] = P(y_- \mid \mathbf{o}) \ell_{\mathrm{fs}}, \quad
\mathbb{E}[\ell(y_-,y)\mid\mathbf{o}] = P(y_+ \mid \mathbf{o}) \ell_{\mathrm{fr}}.
\]
Choosing the action with lower expected loss yields the stated inequality.
\end{proof}

\paragraph{Naive Bayes interpretation.}
Assume audit outcomes are conditionally independent given the true label:
\begin{equation}
P(\mathbf{a} \mid y) = \prod_{i=1}^{k} P(a_i \mid c_i, y).
\end{equation}
Then
\begin{equation}
\frac{P(y_+ \mid \mathbf{a})}{P(y_- \mid \mathbf{a})}
=
\frac{P(y_+)}{P(y_-)} \prod_{i=1}^{k}
\frac{P(a_i \mid c_i, y_+)}{P(a_i \mid c_i, y_-)}.
\end{equation}

For critical conditions, supported audits contribute likelihood ratios greater than one, while contradicted audits contribute ratios less than one. Requiring unanimous support for all critical conditions ensures the posterior odds exceed the asymmetric loss threshold, while a single contradicted critical condition is sufficient to reverse the decision. The aggregation rules thus implement a conservative approximation to Bayes-optimal behavior under asymmetric costs.

\subsection{Component Interaction and Failure Modes}
\label{app:interaction}

The formal properties above rely on the interaction of all pipeline components.

\emph{Decomposition without auditing} produces unverifiable conditions and confidence scores uncorrelated with correctness, violating rank-calibration.

\emph{Auditing without decomposition} collapses the reasoning structure to a single condition, eliminating factorized evidence assessment.

\emph{Abstention without structured confidence} breaks the connection between thresholding and risk reduction.

These theoretical observations are consistent with the ablation results reported in the main paper.

\subsection{Relation to Conformal Prediction}
\label{app:conformal}

The threshold-sweeping approach used to generate risk--coverage curves is related to split-conformal prediction, where calibration sets are used to guarantee coverage or risk bounds. While our framework does not provide distribution-free guarantees, monotonicity of selective risk and finite-sample concentration provide structural and statistical reliability. Incorporating conformal calibration into the pipeline is a natural extension.

\subsection{Scope and Assumptions}
\label{app:scope}

The theoretical results rely on the following assumptions:
\begin{enumerate}
\item Fixed condition decomposition and criticality assignment.
\item Approximate independence of audit outcomes.
\item Empirical rank-calibration of the confidence score.
\item Asymmetric loss structure reflecting scientific practice.
\item Naive Bayes independence used only for interpretive analysis.
\end{enumerate}

These results characterize the pipeline under stated assumptions, while empirical evaluation provides primary evidence of robustness in practice.

\end{document}